\documentclass[conference]{IEEEtran}

\usepackage{cite}
\usepackage{amsmath,amssymb,amsfonts}
\usepackage{algorithmic}
\usepackage{graphicx}
\usepackage{textcomp}
\usepackage{pdfpages}  
\usepackage{xcolor}
\def\BibTeX{{\rm B\kern-.05em{\sc i\kern-.025em b}\kern-.08em
    T\kern-.1667em\lower.7ex\hbox{E}\kern-.125emX}}

\usepackage{amsthm}  
\usepackage{color}
\usepackage{comment}
\usepackage{xspace}
\usepackage{thmtools}  
\usepackage{xcolor}  
\usepackage{multirow}  
\usepackage{url}

\usepackage[capitalize]{cleveref}
\crefname{section}{Sec.}{Secs.}
\Crefname{section}{Section}{Sections}
\Crefname{table}{Table}{Tables}
\crefname{table}{Tab.}{Tabs.}
\Crefname{figure}{Fig.}{Fig.}
\crefname{figure}{Fig.}{Fig.}

\newtheorem{corr}{Corollary}
\newtheorem{lem}{Lemma}

\theoremstyle{definition}

\usepackage{stmaryrd}       
\usepackage{xspace}
\makeatletter
\DeclareRobustCommand\onedot{\futurelet\@let@token\@onedot}
\def\@onedot{\ifx\@let@token.\else.\null\fi\xspace}

\makeatother

\newcommand{\mm}{MaxMin}
\newcommand{\maxmin}{MaxMin\xspace}
\newcommand{\abs}{absolute value\xspace}
\newcommand{\universal}{\emph{universal}\xspace}
\newcommand{\expressable}{\textit{1-CPWL}\xspace}
\newcommand{\express}{\emph{express}\xspace}
\newcommand{\onedim}{$1$-dimensional\xspace}

\newcommand{\CPWL}{continuous, piecewise-linear}
\newcommand{\ols}{$1$-Lipschitz\xspace}
\newcommand{\nfun}{$\mathcal{N}$-function\xspace}
\newcommand{\nact}{$\mathcal{N}$-activation\xspace}
\newcommand{\titlenact}{N-Activation}

\newcommand{\AOL}{AOL\xspace}

\newcommand{\CPL}{CPL\xspace}
\newcommand{\SOC}{SOC\xspace}

\newcommand{\absidinit}{AbsId-Initialiation\xspace}
\newcommand{\zeroinit}{Zero-Initialiation\xspace}
\newcommand{\randominit}{Random-Initialiation\xspace}

\newcommand{\CIFARS}{CIFAR-10\xspace}

\newcommand{\tinyIN}{Tiny ImageNet\xspace}

\newcommand{\ZeroChannelConcatenation}{\operatorname{ZeroChannelConcatenation}}
\newcommand{\FirstChannels}[1]{\operatorname{FirstChannels}(#1)}
\newcommand{\ConcatenationPooling}{\operatorname{ConcatenationPooling}}
\newcommand{\ConvBlock}[1]{\operatorname{ConvBlock}(#1)}

\newcommand{\cra}{certified robust accuracy\xspace}
\newcommand{\Cra}{Certified robust accuracy\xspace}

\newcommand{\tableheader}{Certified Robust Accuracy}

\renewcommand{\vec}[1]{\boldsymbol{#1}}

\newcommand{\mrtimes}[1]{\multirow{ 2}{*}{$#1\times \Big\{$}}

\newcommand{\citet}[1]{\cite{#1}}
\newcommand{\citep}[1]{\cite{#1}}

\begin{document}

\title{1-Lipschitz Neural Networks are more expressive with N-Activations}

\author{
    \IEEEauthorblockN{Bernd Prach}
    \IEEEauthorblockA{
        \textit{Institute of Science and Technology Austria (ISTA)} \\
        Klosterneuburg, Austria \\
        bprach@ist.ac.at
    }
    \and
    \IEEEauthorblockN{Christoph H. Lampert}
    \IEEEauthorblockA{
        \textit{Institute of Science and Technology Austria (ISTA)} \\
        Klosterneuburg, Austria \\
        chl@ist.ac.at
    }

}
    
\maketitle
\IEEEpeerreviewmaketitle

\begin{abstract}
    A crucial property for achieving secure, trustworthy and interpretable deep learning systems is their robustness: small changes to a system's inputs should not result in large changes to its outputs.
    Mathematically, this means one strives for networks with a small \emph{Lipschitz constant}. 
    Several recent works have focused on how to construct such \emph{Lipschitz networks}, 
    typically by imposing constraints on the weight matrices.
    In this work, we study an orthogonal aspect, namely the role of the activation function. We show that commonly used activation functions, such as \emph{MaxMin}, as well as all piece-wise linear ones with two segments unnecessarily restrict the class of representable functions, even in the simplest one-dimensional setting.
    We furthermore introduce the new \nact function that is provably more expressive 
    than currently popular activation functions.
    %
\end{abstract}


\section{Introduction}
In recent years, deep neural networks have achieved
state-of-the-art performances in most computer vision tasks.
However, those models have a (maybe surprising) shortcoming:
very small changes to an input of a model, 
usually invisible to the human eye,
can change the output of the model drastically.
This phenomenon is known as \emph{adversarial examples}
\cite{adversarial_examples_2014_szegedy}.

The fact that humans can easily create small adversarial perturbations and fool deep networks makes it hard to trust those models when used in the real world.
Furthermore, it could lead to security issues in high-stakes
computer vision tasks, for example in autonomous driving. 
Finally, the existence of adversarial examples
is a big problem for interpretability,
and it makes it very hard to find explanations
that are truthful to the model and at the same
time useful to humans 
(see e.g. \cite{explainable_requires_robustness_2021_Leino}).
Therefore, there has been a large interest in training
deep networks that are more robust to small changes in 
the input.

One way of making models more robust is 
\emph{adversarial training} \cite{adversarial_training_2015_goodfellow}.
In adversarial training, adversarial examples are constructed during
training and added to the training set
in order to make the model more robust.
However, whilst adversarial training does make it harder to 
find adversarial examples at inference time,
it lacks guarantees of whether they do exist.
We do believe that having some guarantees can be very important.
Another method that does give guarantees is 
\emph{randomized smoothing} \cite{randomized_smoothing_2019_cohen}. 
In randomized smoothing, at inference time, one adds random noise
to an input before classifying it. This is done thousands
of times for each image, and the majority vote is taken.
This procedure does guarantee the robustness of the system,
however, it comes with a huge computational overhead at inference time.

Therefore, we believe that the most promising way to obtain
robust models is by restricting the Lipschitz constant
of a model.
Early approaches tried to accomplish this using regularization
\cite{Cisse_2017_ICML}
or weight clipping \cite{arjovsky2017wasserstein}.
However, these were replaced by approaches that actually give
guarantees of robustness
\cite{Tsuzuku_2018_NIPS,Anil_2019_ICML,Li_2019_NIPS_BCOP,Huang_2020_CVPR,Trockman_2021_ICLR,Singla_2021_ICML,Leino_2021_ICML,Prach_2022_ECCV}.
These models usually restrict the Lipschitz constant
of a model to be at most $1$ by restricting the Lipschitz constant
of each individual layer.

Such \emph{\ols models} are provably robust,
however, their empirical performance is much worse than
what we might hope for.
This is even true on very simple tasks 
and for fairly small perturbations of the input.
This raises the question 
if this task is just inherently very difficult,
or if there is a theoretical shortcoming in our current approaches 
that prevents them from reaching better performance.

In this paper, we will show that there actually is a shortcoming.
We prove that the activation function commonly used in
such \ols networks, \maxmin~\cite{Anil_2019_ICML},
does not allow them to express all
functions that we might expect such networks to
be able to express.
In particular, we show (theoretically and experimentally) that 
\ols networks with \maxmin activation functions
can not even express a very simple \onedim (\ols) function!

To overcome the shortcoming, we introduce an activation function that 
(provably) can express any reasonable \onedim function.
We call the it the \emph{\nact}.
Apart from providing a theoretical guarantee,
we also show empirical results.
In particular, we show that the \nact is a competitive
replacement for the \maxmin activation in the task
of certified robust classification. \medskip

In summary, we show a shortcoming with currently popular
activation functions, both empirically and experimentally.
Then we propose an activation function that provably
overcomes this limitation.

\section{Background and Definitions}\label{sec:background}
Before discussing our main results and related work, we introduce some notation and terminology. 
We call a layer, a network or a function, 
$f : \mathbb{R}^n \rightarrow \mathbb{R}^m$, \emph{\ols}%
\footnote{We will exclusively consider the Lipschitz property 
with respect to the $L_2$-norm in this work.} if
\begin{align}
\forall x , y \in \mathbb{R}^n, \qquad
    \|f(x) - f(y)\|_2 \le \|x-y\|_2,
    \label{eq:ols}
\end{align}
Examples of \ols functions are 
linear functions $x\mapsto Ax$ for which the matrix $A$ has \emph{operator  norm} no bigger than $1$, as well as many common activation functions, such as 
\emph{ReLU}, \emph{\abs}, 
or \emph{\maxmin} \cite{Anil_2019_ICML} (see \Cref{eq:maxmin}). 
From the definition in \Cref{eq:ols}
it follows directly that the concatenation of \ols functions is \ols again. 
This provides an immediate mechanism to construct \ols neural networks: one simply alternates linear operations with suitably constrained weights matrices with any choice of \ols activation function.
We will use shorthand notation 
\emph{\mm}-networks for such \ols networks with \emph{\mm}-activation function, etc. 

In this work, we are particularly interested in the predictive power of \ols networks. 
Clearly, there are principled restrictions on the kinds of functions that these can express.
In particular, if the activation function is continuous and piece-wise linear, the whole network will also only be able to express continuous piece-wise linear functions.
Also, by construction, the network can of course only express 1-Lipschitz functions.
Within those two restrictions, though, the networks should ideally be able to express any function.

Formally, we call a function \expressable, if it is \ols, continuous and piece-wise linear with finitely many segments. 
We say an activation function can \express a function $f$
if we can write $f$ using only \ols linear operations and the activation function.
We call an activation function that can \express all \expressable functions
\expressable-\universal (short: \universal).

With this notation, the key questions 
we are interested in can be formulated as: 
\textbf{Do the classes of expressible functions differ between different activation functions?} and 
\textbf{Does there exist a \universal activation function?}

\section{Related work}\label{sec:relatedwork}
\subsection{\ols linear layers}
A number of designs for networks with Lipschitz constant $1$
have been proposed in recent years, 
including~\citet{Prach_2022_ECCV,Meunier_2022_ICML,brau_2023_ugnn,Singla_2021_ICML,Trockman_2021_ICLR,Leino_2021_ICML,Anil_2019_ICML}.
They often focus on how to parameterize the linear layers of a network efficiently. 
We will describe details here for three particular methods,
that we use in this paper to evaluate our results. \medskip

The \emph{Almost Orthogonal Lipschitz (\AOL)} method \cite{Prach_2022_ECCV}
uses a rescaling-based approach in order to make fully connected
as well as convolutional layers \ols.
Mathematically, for a parameter matrix $P$, they define a diagonal matrix $D$ with
\begin{align}
    D_{ii} = \left( \sum_j | P^\top P |_{ij} \right)^{-1/2}.
\end{align}
Then they show that the linear layer given as $f(x) = PDx + b$ is \ols.
They also extend this bound and propose a channel-wise rescaling
that guarantees convolutions to be \ols. \medskip

\emph{Convex Potential Layers (CPL)} \cite{Meunier_2022_ICML}
is another method of parameterizing \ols layers.
The proposed \ols layer is given as 
\begin{align}
    f(x) = x - \frac{2}{\|P\|_2^2} P^\top \sigma \left( Px + b \right),
\end{align}
for $P$ a parameter matrix, 
and $\sigma$ a non-decreasing \ols function, usually ReLU.
Here, the spectral norm is usually computed using power iterations. \medskip

A third method, \emph{Skew Orthogonal Convolutions (\SOC)}
\cite{soc}
produces layers with an orthogonal Jacobian matrix.
The \SOC layer can be written as
\begin{align} \label{eq:soc}
    f(x) = x + \frac{P \star x}{1!} + \frac{P \star^2 x}{2!} + \dots 
        + \frac{P \star^k x}{k!} + ..,
\end{align}
where $P \star^k$ denotes application of $k$ convolutions
with kernel $P$.
When the kernel, $P$, is skew-symmetric,
then the Jacobian of this layer is orthogonal.
For their experiments, the authors use a truncated version of \Cref{eq:soc},
with 5 terms during training and 12 during inference. \medskip

For an overview of further methods of parameterizing \ols linear layers
see e.g. \cite{comparison_2023_unpublished}. \medskip

\subsection{Shortcomings of the \ols setup}

Whilst a lot of work has focused on parameterizing \ols linear layers so far, 
much less work has focused on 
the problems and shortcomings of \ols networks in general,
as well as shortcomings with
the architectural decisions often made in current \ols networks.

Some works have looked into general limitations of robust networks.
For example, \cite{bubeck_2021_universal,bombari_2023_beyond_universal_law} showed that
in order to perfectly interpolate noise in the training data
in a robust way,
one does require many times more parameters than data points.
However, we are not interested in achieving zero training loss,
but we care about the generalization performance,
so their theory is not applicable in our scenario.
Similarly, \cite{corner_2023_leino} showed that there
exist data distributions
for which a simple robust classifier exists,
however,
any \ols score-based classifier requires 
to fit a much more complicated function
with many more non-linearities.

There is some work (such as \cite{universal_2019_cohen})
that provides results for different norms (like $L_\infty$),
however, in this paper, we are only interested in 1-Lipschitz networks
with respect to the $L_2$ norm.

As shown by \cite{dense_2020_Eckstein}
any 1-Lipschitz function (with one output) 
can be approximated (arbitrarily well) by a neural network 
in a way that the approximation is also 1-Lipschitz.
However, we are more interested in \ols neural network
with the commonly used design where every single
layer is \ols,
and all intermediate representations are 
\ols transformations of the inputs as well.

For the $L_2$ norm and
for networks consisting purely of \ols layers,
early on, \citet{Anil_2019_ICML} observed that the 
otherwise popular ReLU activation function is not a good choice in this context.
They showed that \ols ReLU networks are not able to fit even some simple functions, 
such as the absolute value function.
As an improvement, the authors suggested \emph{GroupSort} with its special
case of \emph{\mm}, given as
\begin{align} \label{eq:maxmin}
    \operatorname{MaxMin}(\left(\begin{matrix} x \\ y \end{matrix} \right) )
        = \left(\begin{matrix} \max (x, y) \\ 
            \min (x, y) \end{matrix} \right),
\end{align}
and proved a result of universal approximation.
However, for this result they constraint the operator norm 
of the weight matrices using the $L_\infty$ norm
instead of the $L_2$ norm.

Other activations such as \emph{Householder activations} \cite{householder_2021_singla}
have been proposed to improve results with \ols networks,
however, e.g. \cite{Splines_2022_ArXiv} showed
that the corresponding networks do represent the same set of functions as \maxmin-networks.

Finally, as a very promising direction, \cite{splines_2020_ieee} used (parameterized) linear splines as an activation function in \ols networks.
Building on this work, \cite{Neumayer_2022_ArXiv} showed that \ols linear splines with $3$ linear regions are \universal{} as an activation function in $1$ dimension.
However, training spline-networks with a general parameterization seems difficult in practice, 
and \citet{Splines_2022_ArXiv} use three different learning rates 
as well as auxiliary loss functions in order to do that.
As one of our contributions, we will show that we can further restrict the set of splines that can be learned to have a specific structure, without restricting the class of functions the network can express.

\section{Theoretical Results}~\label{sec:method}
In this section, we present our main results, namely that already for functions with one-dimensional input and output the answer to both of the questions we asked in Section~\ref{sec:background} is \emph{yes}.
First, in Section~\ref{subsec:minmax} we show that the commonly used activation functions are not able to express all \expressable functions. 
Then, in Section~\ref{subsec:N}, we introduce a new activation function, called $\mathcal{N}$-activation, and show that \nact networks are indeed able to express any \expressable one-dimensional function.

\subsection{Limitations of the existing activation functions}\label{subsec:minmax}
Our first result is that a whole class of activation functions,
namely \expressable{} functions with 2 segments are not \universal:
\begin{restatable}{thm}{twopiecebad} \label{thm:2piecebad}
    No 2-piece \expressable activation is \universal.
\end{restatable}
In order to prove this theorem,
it is enough to show that
there exists a 1-dimensional, \expressable{} function
that cannot be expressed by any network with such activation functions.
There is in fact a simple function that can not be expressed;
we call it the \nfun. It is given as:
\newcommand{\half}{\frac{1}{2}}
\begin{align}
    \mathcal{N}(x) = \left\{\begin{array}{llrll}
        x+1 & \text{ ... } &            & x  & \le -\half \\
        -x  & \text{ ... } & -\half \le & x  & \le  \half \\
        x-1 & \text{ ... } &  \half \le & x.
    \end{array} \right.
\end{align}
It is visualized in \Cref{fig:n_function}.
It has the property that the average gradient between $-x$ and $+x$ converges
to $1$ when $x$ goes to infinity, and we show that this property
cannot be achieved with \abs networks.
This then proves \cref{thm:2piecebad} since any 2-piece
\expressable{} activation function can be written
as a linear (and \ols) combination of an
\abs-activation and the identity map
(with suitable biases).

\begin{figure}[t]
\centering
\includegraphics[width=6cm]{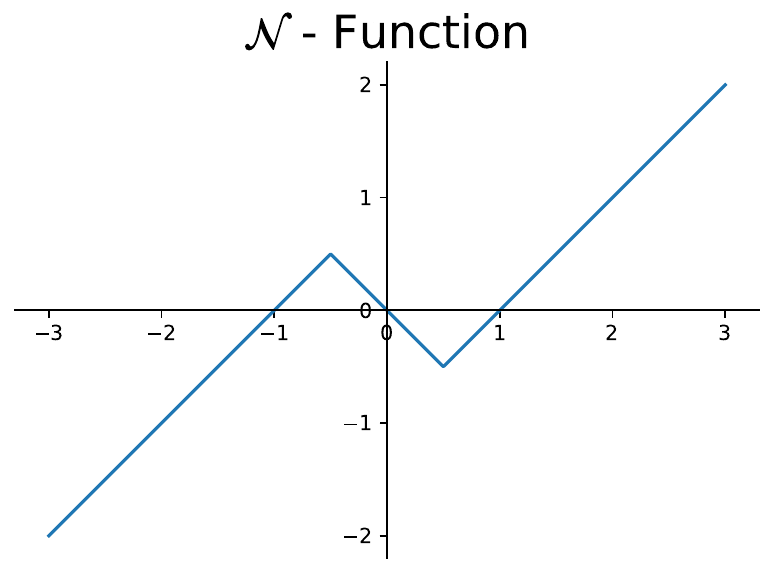}
\caption{
    A plot of the \nfun.
}
\label{fig:n_function}
\end{figure}

\Cref{thm:2piecebad}
implies that activation functions 
such as ReLU, leaky ReLU and \abs are not \universal.
Furthermore, in \Cref{eq:maxmin_equals_abs_id} we show that 
we can write \mm-activations using orthogonal matrices
and a combination of identity connections and \abs{} activations. 
Therefore, \Cref{thm:2piecebad} also
implies that \mm-activations are not \universal:

\begin{restatable}{corr}{maxminbad} \label{thm:maxminbad}
    \mm-activations are not \universal.
\end{restatable}

\begin{figure}[t]
\centering
\includegraphics[width=6cm]{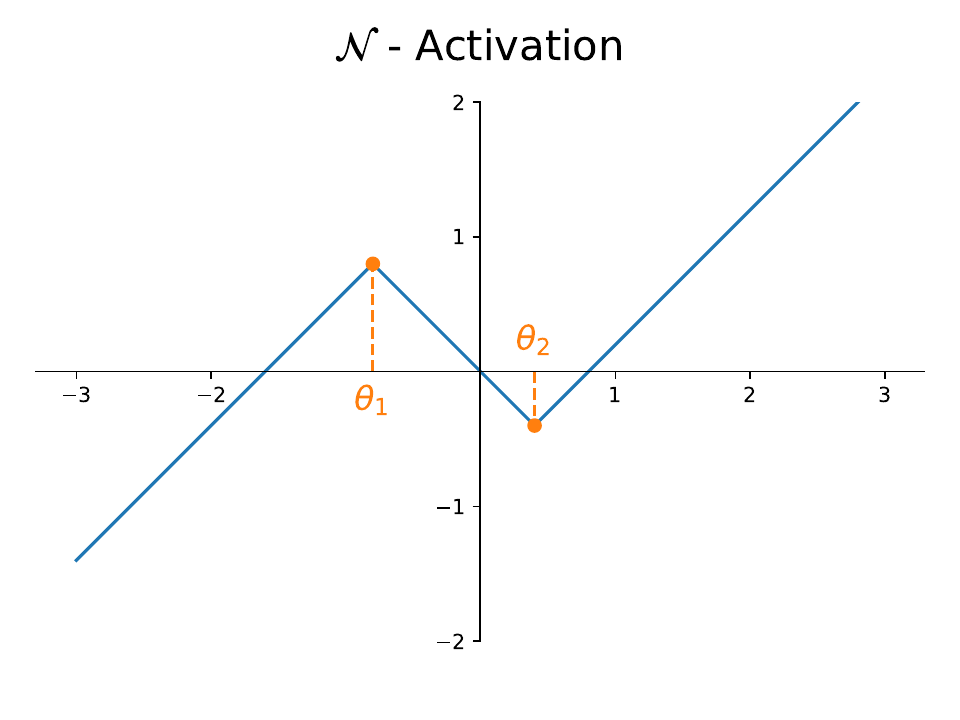}
\caption{
    A plot of the \nact with parameters $\theta_1$ and $\theta_2$.
}
\label{fig:n_activation}
\end{figure}

\subsection{The \titlenact{} is \universal in 1D.}~\label{subsec:N}
In order to be able to express the \nfun{}, 
we will define an activation that is more expressive
than \mm{}.
Inspired by the \nfun{}, we define the (parameterized) \nact{}:

\newcommand{\thmax}{\theta_{\max}}
\newcommand{\thmin}{\theta_{\min}}
\newcommand{\nactdef}{
\begin{align}
    \mathcal{N}(x; \theta_1, \theta_2)
    = \left\{\begin{array}{lllll}
        x-2\thmin  & \text{ .. } &               &x  &\le \thmin \\
        -x         & \text{ .. } & \thmin \le  &x  &\le \thmax \\
        x-2\thmax  & \text{ .. } & \thmax \le  &x,
    \end{array} \right.
\end{align}
}
\nactdef{}%
where $\thmax = \max\left(\theta_1, \theta_2\right)$
and $\thmin = \min\left(\theta_1, \theta_2\right)$.
For a visualization of the \nact see \Cref{fig:n_activation}.

In our arguments, we will allow $\theta_1$ to be $-\infty$. 
In this case, the \nact{} is equal to the \abs activation when $\theta_2=0$, 
so this is equivalent to allowing a network to use a combination of \nact{}s 
(with $\theta_1$ and $\theta_2$ finite) and \abs activations.
Note that for bounded inputs, setting $\theta_1$ to
some negative value with high magnitude will have the
same effect as (theoretically) setting it to $-\infty$.
Also note that setting $\theta_1 = \theta_2 = 0$ allows us 
to express the identity map using an \nact{}.

A network using the \nact{} is trivially able to fit
the \nfun{}, but furthermore, the \nact{} is \universal:
It can express any \onedim, \expressable{} function:

\begin{restatable}{thm}{nactuniveral} \label{thm:nactuniversal}
    \nact{}s are \universal{} in 1D.
    Specifically, any \expressable{}, \onedim{} function
    with $k$ non-linearities
    can be represented by a network of width $2$ consisting of $(k+5)$
    linear layers and $(\frac{3}{2}k + 5)$ \nact{}s.
\end{restatable}

The proof can be found in \cref{sec:nactgood}.
We first show that we can express \emph{increasing} \expressable{} functions
using $k$ linear layers and $(k-1)$ \nact{}s.
Then we show that each function with $2l$ local extreme points
can be expressed by applying the \nact{} to some function with
$2l-2$ local extreme points, 
and that allows us to use induction to show that we can express
functions with any number of extreme points.
Finally, we need $3$ \abs activations to adjust the slope
of a function before the first and after the last non-linearity.

\section{Experimental setup}

\subsection{Fitting the \nfun{}}

We first designed a toy experiment to show that
not only is it theoretically impossible to exactly
express the \nfun,
but also approximating the \nfun{} well
with \mm-networks seems impossible, 
even when the inputs are bounded.
To show this, we 
train networks to fit the \nfun{} on the interval $[-3, 3]$.
As a training set, we sample 1000 points uniformly from [-3, 3], 
and obtain the targets by applying the \nfun.
We try to fit this data using networks with different activation functions.
Each network consists of 3 dense layers with activation functions 
in between.
%
We use \AOL to constrain the linear layers to be \ols.
The layers are of width 40. 
We optimize each network using Nesterov SGD with learning rate
0.01 and momentum of 0.9, with a batch size of 100
for 1000 epochs.
We fit each of the networks to our training set using mean squared error,
and also report the mean squared error on the training set for the different models.

\subsection{Certified Robust Classification}
In order to confirm that our method does not
only work in theory and in constructed toy
problems, we also compare our activation
function to \mm{} in the standard task
for evaluating 1-Lipschitz models: 
Certified Robust Classification.

In certified robust classification 
we do not only aim for models that classify inputs correctly.
We furthermore want to be able to guarantee that no
perturbations with magnitude bounded by some predefined value
can change the class a model predicts.
Mathematically, an input is classified \emph{certifiably robustly}
with radius $\epsilon$ by a model
if we can show that no perturbation of size
at most $\epsilon$ can change the prediction
of the model.
The \cra measures the proportion of test
examples that are classified correctly
as well as certifiably robustly.
For \ols models we can use the \ols property
to easily evaluate if an example is classified
certifiably robustly:
We only need to consider the difference
between the highest two scores generated by the model.
If it is large enough we can be sure that no
small perturbation can change the order of those
scores, and therefore the class
predicted by the model can not change.
For details see \cite{Tsuzuku_2018_NIPS}.

It is not obvious that higher expressiveness
leads to better \cra. However, we do use a loss
function that encourages high robust accuracy on the training data.
We hope that the fact that \nact-networks are
more expressive than \maxmin-networks will allow
\nact-networks to obtain a lower loss.
In \ols networks, it generally seems the case
that fitting the training data is a bigger
challenge than generalizing. Therefore we
hope that a more expressive model will also
have a higher \cra.

In this paper, our main focus is on 
\cra with $\epsilon=36/255$.
We optimize the hyperparameters 
(specifically learning rate, see \Cref{table:learningrates})
with respect
to this metric, and also report
results with this value of $\epsilon$
in our figures. \medskip

For our experimental setup, we largely follow
\cite{comparison_2023_unpublished},
including architecture, loss function hyperparameters
as well as implementations.
We describe the details of our experiments below. \medskip

\setlength{\tabcolsep}{4pt}
\begin{table}[t]
\begin{center}
\caption{
    Learning rates used for different layer types.
}
\label{table:learningrates}
\begin{tabular}{cl}

\noalign{\smallskip} \hline \noalign{\smallskip}
Layer Type & Learning Rate \\
\noalign{\smallskip} \hline \noalign{\smallskip}
\AOL & $10^{-1.6}$ \\
\CPL & $10^{-0.4}$ \\
\SOC & $10^{-1.0}$ \\
\noalign{\smallskip} \hline \noalign{\smallskip}

\end{tabular}
\end{center}
\end{table}
\setlength{\tabcolsep}{1.4pt}

\subsubsection{Architecture}
Our architecture is similar to the
medium-sized (M) architecture from 
\cite{comparison_2023_unpublished}.
It is a relative standard convolutional architecture,
where the number of channels doubles whenever the
resolution is reduced.
All convolutional layers have the same input and
output size, 
which is achieved by using $\ZeroChannelConcatenation$ layers
to append channels with value $0$ to an input,
and $\FirstChannels{c}$ layers,
that return the first $c$ channels of the input and ignore the rest.
We use $\ConcatenationPooling$ to reduce the image size.
In our architecture the $\ConcatenationPooling$
takes all values from a
$2 \times 2 \times c$ patch and stacks them into one vector of size
$1 \times 1 \times 4c$.
Note that this layer is very similar to the \emph{PixelUnshuffle} layer
of PyTorch, however, it differs in the order of channels of the output.
For an overview of the architecture see
\Cref{table:architecture,table:conv_block}. \medskip

\setlength{\tabcolsep}{4pt}
\begin{table}[t]
\begin{center}
\caption{
    \textbf{Architecture used},
    depending on width parameter $w$
    and number of classes $c$.
    A description of the \emph{ConvBlock} can be found 
    in \Cref{table:conv_block}.
}
\label{table:architecture}
\begin{tabular}{rl|l}
\hline\noalign{\smallskip}
& Layer name & Output size \\
\noalign{\smallskip}
\hline
\noalign{\smallskip}

& $\ZeroChannelConcatenation$ & $32 \times 32 \times w$   \\
& $1 \times 1$ Convolution          & $32 \times 32 \times w$ \\
& Activation                        & $32 \times 32 \times w$ \\

& $\ConvBlock{k=3}$                & $16 \times 16 \times 2w$ \\
& $\ConvBlock{k=3}$               & $8 \times 8 \times 4w$ \\
& $\ConvBlock{k=3}$               & $4 \times 4 \times 8w$ \\
& $\ConvBlock{k=3}$               & $2 \times 2 \times 16w$ \\
& $\ConvBlock{k=1}$                & $1 \times 1 \times 32w$ \\

& $1 \times 1$ Convolution          & $1 \times 1 \times 32w$ \\
& $\FirstChannels{c}$     & $1 \times 1 \times c$ \\
& Flatten                           & $c$ \\

\noalign{\smallskip}
\hline
\end{tabular}
\end{center}
\end{table}
\setlength{\tabcolsep}{1.4pt}

\setlength{\tabcolsep}{4pt}
\begin{table}[t]
\begin{center}
\caption{
    Convolutional block.
    We show $\ConvBlock{k}$, 
    for input size $s \times s \times w$
    and kernel size $k$.
}
\label{table:conv_block}
\begin{tabular}{rl|l}
\hline\noalign{\smallskip}
& Layer name & Output size \\
\noalign{\smallskip}
\hline
\noalign{\smallskip}

\mrtimes{5}
& $k \times k$ Convolution          & $s \times s \times w$ \\
& Activation                        & $s \times s \times w$ \\
& $\FirstChannels{w/2}$             & $s \times s \times w/2$ \\
& $\ConcatenationPooling$           & $s/2 \times s/2 \times 2w$ \\

\noalign{\smallskip}
\hline
\end{tabular}
\end{center}
\end{table}
\setlength{\tabcolsep}{1.4pt}

\subsubsection{Loss function}
We use the loss function proposed in \cite{Prach_2022_ECCV}.
Like \cite{comparison_2023_unpublished} we set
the \emph{offset} parameter $u$ to $2 \sqrt{2} \epsilon$,
for $\epsilon=36/255$, and the \emph{temperature}
parameter $t$ to $1/4$.
The loss function is defined as
\begin{align}
    \mathcal{L}(\vec{x}, \vec{y}) 
    = \operatorname{crossentropy}
    \left(\vec{y}, \ 
    \operatorname{softmax}\left(\frac{\vec{x} - u\vec{y}}{t}\right)
    \right)t,
\end{align}
for $\vec{x}$ the vector of class scores and $\vec{y}$ the
one-hot encoding of the target labels.

\medskip

\subsubsection{Optimization}
We use a setup similar to \cite{comparison_2023_unpublished}.
We use SGD with a momentum of 0.9 for all experiments. We also use a learning rate schedule. We choose to use \emph{OneCycleLR}, as described by \cite{onecyclelr_2019_smith}, with default values as in \emph{PyTorch}.
We train for 1000 epochs and
set the batch size to 256 for all experiments.
We subtract the dataset means before training.
As data augmentation, we use random crops and random flips
on CIFAR, and \emph{RandAugment} \cite{randaugment_2020_cubuk} on \tinyIN.
\medskip

\subsubsection{Hyperparameter search}
We first did some preliminary search to figure out
reasonable learning rates for each layer type.
We did this search by (uniformly) randomly sampling
log-learning-rates, and evaluating the different
models on \CIFARS in terms of the \cra (for $\epsilon=36/255$).
Then we selected the learning rate (for all datasets) based
on those results.
The learning rates we choose can be found in
\cref{table:learningrates}. \medskip

\subsubsection{Initialization of the \nact}
\label{sec:nact_init}
Initialization can be a very important aspect of \ols networks,
since those usually do not have residual connections or batch
normalization layers that help against vanishing gradients.
We also found that the initialization of the \nact is crucial
to ensure good performance.
We took inspiration from the \maxmin activation.
In particular, we noticed that the \maxmin activation is equivalent
(up to a rotation of input and output) 
to applying absolute value and identity map in an alternating way
(see \Cref{eq:maxmin_equals_abs_id}).
We initialize the \nact to express this activation:
For each pair of channels, 
for the first channel we set $\theta_1 = -100$ and $\theta_2 = 0$
(to initialize as the \abs function),
for the second channel we set $\theta_1 = \theta_2 = 0$
(to initialize as the identity map). \medskip

\subsubsection{Learning rate for the \nact}
Unfortunately, we observed that using the same learning rate
for the parameters of linear layers and for the parameters
of the \nact does hurt performance.
In order to overcome that we rescale the parameters of the
\nact by a factor of $1/10$,
effectively reducing the learning rate.\medskip

\begin{figure}[pt]
\centering
\includegraphics[height=40mm]{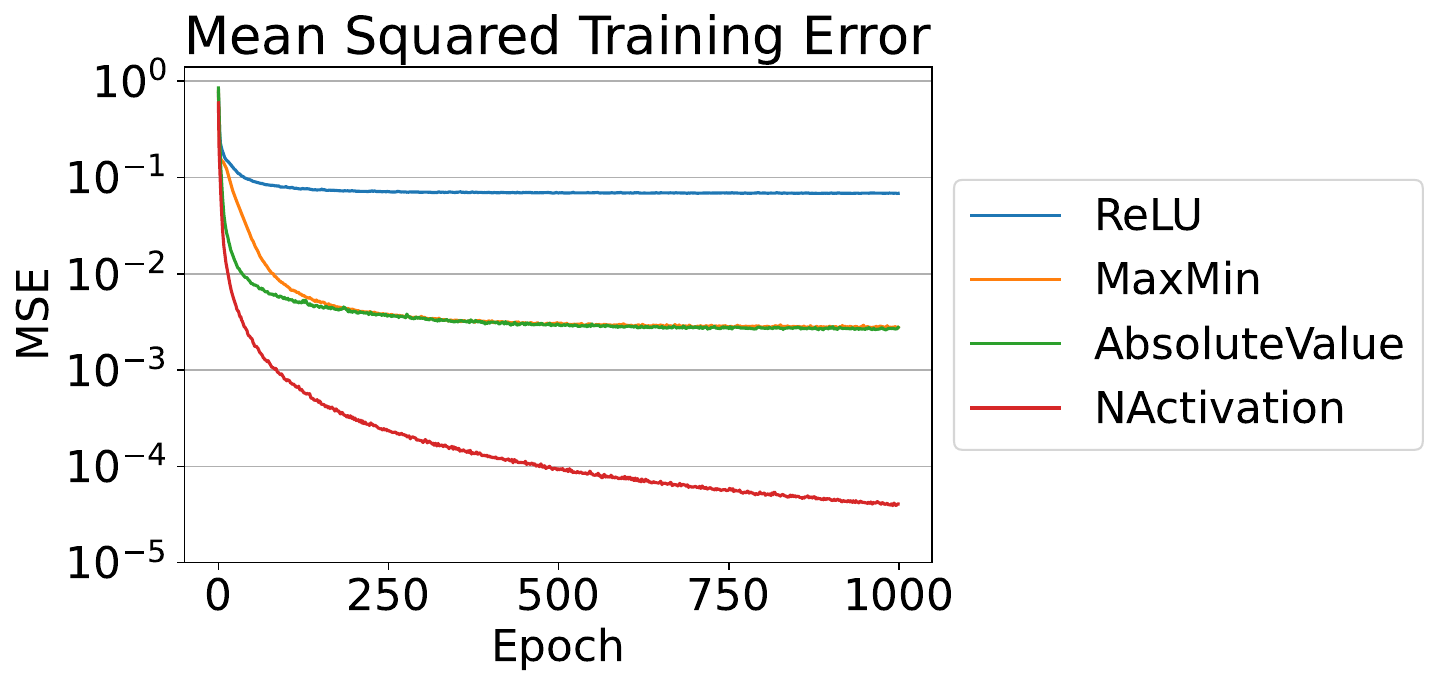}
\caption{
    Mean squared error on the training set reported
    for \ols \AOL networks with different activation functions
    for fitting the \nfun.
}
\label{fig:mse_figure}
\end{figure}

\begin{figure}[t]
\centering
\includegraphics[width=\columnwidth]{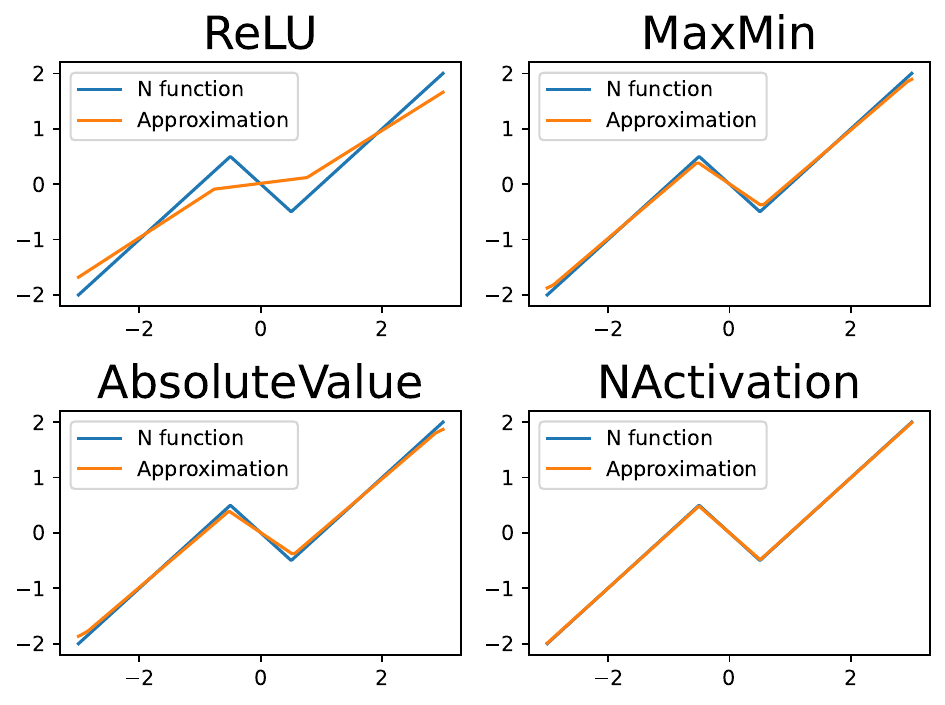}
\caption{
    ReLU networks, \maxmin networks and \abs networks can not fit
    the \nfun, whereas \nact networks can!
}
\label{fig:approximations}
\end{figure}

\begin{figure*}[pt]
\centering
\includegraphics[width=\textwidth]{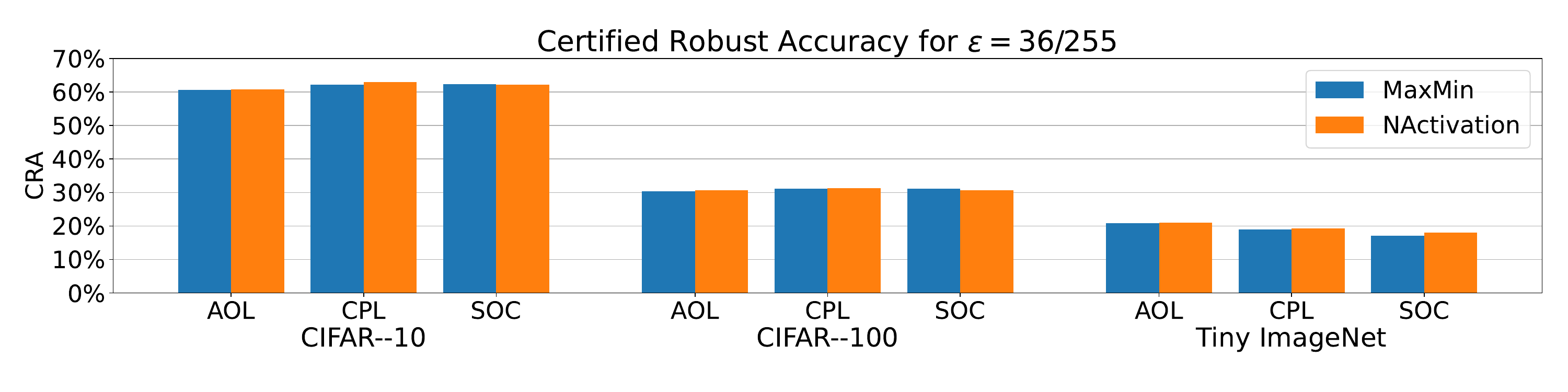}
\caption{
    \Cra on different datasets, for different \ols layers.
    \maxmin and \nact compared.
}
\label{fig:results}
\end{figure*}

\newcommand{\best}[1]{\textbf{#1}}

\begin{table*}[t]
\setlength{\tabcolsep}{4pt}
\begin{center}
\caption{
    Test set accuracy and \cra for different thresholds $\epsilon$.
    We compare the \nact{} to the \mm-activation in different settings.
    Better results are marked in bold.
}
\label{table:results}
\begin{tabular}{cccrrrrr}
\hline\noalign{\smallskip}
Dataset     & Layer     & Activation    & Accuracy  & \multicolumn{4}{c}{\tableheader} \\ %
            & Type     &               &           & 
$\epsilon=\frac{36}{255}$ & $\epsilon=\frac{72}{255}$ & $\epsilon=\frac{108}{255}$ & $\epsilon=1$ \\%
            


\noalign{\smallskip} \hline \noalign{\smallskip}
\multirow{2}{*}{CIFAR-10} & \multirow{2}{*}{AOL} & MaxMin & $72.7\%$ & $60.6\%$ & $\best{47.0\%}$ & $34.1\%$ & $4.3\%$ \\
 &  & NActivation & $\best{72.9\%}$ & $\best{60.7\%}$ & $46.9\%$ & $\best{34.7\%}$ & $\best{4.5\%}$ \\
\noalign{\smallskip} \hline \noalign{\smallskip}
\multirow{2}{*}{CIFAR-10} & \multirow{2}{*}{CPL} & MaxMin & $74.9\%$ & $62.2\%$ & $47.9\%$ & $35.2\%$ & $4.0\%$ \\
 &  & NActivation & $\best{75.0\%}$ & $\best{63.0\%}$ & $\best{48.6\%}$ & $\best{35.6\%}$ & $\best{4.2\%}$ \\
\noalign{\smallskip} \hline \noalign{\smallskip}
\multirow{2}{*}{CIFAR-10} & \multirow{2}{*}{SOC} & MaxMin & $\best{74.5\%}$ & $\best{62.3\%}$ & $49.2\%$ & $36.0\%$ & $5.0\%$ \\
 &  & NActivation & $\best{74.5\%}$ & $62.2\%$ & $\best{49.6\%}$ & $\best{36.7\%}$ & $\best{5.1\%}$ \\
\noalign{\smallskip} \hline \noalign{\smallskip}
\multirow{2}{*}{CIFAR-100} & \multirow{2}{*}{AOL} & MaxMin & $43.1\%$ & $30.4\%$ & $20.7\%$ & $14.0\%$ & $\best{2.1\%}$ \\
 &  & NActivation & $\best{43.2\%}$ & $\best{30.6\%}$ & $\best{20.9\%}$ & $\best{14.1\%}$ & $\best{2.1\%}$ \\
\noalign{\smallskip} \hline \noalign{\smallskip}
\multirow{2}{*}{CIFAR-100} & \multirow{2}{*}{CPL} & MaxMin & $\best{43.8\%}$ & $31.1\%$ & $21.3\%$ & $14.3\%$ & $2.3\%$ \\
 &  & NActivation & $43.3\%$ & $\best{31.2\%}$ & $\best{21.6\%}$ & $\best{14.8\%}$ & $\best{2.4\%}$ \\
\noalign{\smallskip} \hline \noalign{\smallskip}
\multirow{2}{*}{CIFAR-100} & \multirow{2}{*}{SOC} & MaxMin & $\best{43.0\%}$ & $\best{31.2\%}$ & $\best{21.3\%}$ & $14.3\%$ & $2.1\%$ \\
 &  & NActivation & $42.8\%$ & $30.6\%$ & $21.2\%$ & $\best{14.4\%}$ & $\best{2.3\%}$ \\
\noalign{\smallskip} \hline \noalign{\smallskip}
\multirow{2}{*}{Tiny ImageNet} & \multirow{2}{*}{AOL} & MaxMin & $30.4\%$ & $20.8\%$ & $\best{13.6\%}$ & $9.0\%$ & $1.0\%$ \\
 &  & NActivation & $\best{30.8\%}$ & $\best{20.9\%}$ & $13.5\%$ & $\best{9.1\%}$ & $\best{1.1\%}$ \\
\noalign{\smallskip} \hline \noalign{\smallskip}
\multirow{2}{*}{Tiny ImageNet} & \multirow{2}{*}{CPL} & MaxMin & $\best{28.6\%}$ & $19.0\%$ & $\best{12.4\%}$ & $\best{8.0\%}$ & $\best{0.8\%}$ \\
 &  & NActivation & $28.3\%$ & $\best{19.3\%}$ & $\best{12.4\%}$ & $\best{8.0\%}$ & $\best{0.8\%}$ \\
\noalign{\smallskip} \hline \noalign{\smallskip}
\multirow{2}{*}{Tiny ImageNet} & \multirow{2}{*}{SOC} & MaxMin & $26.4\%$ & $17.2\%$ & $\best{10.9\%}$ & $6.9\%$ & $0.6\%$ \\
 &  & NActivation & $\best{27.1\%}$ & $\best{18.1\%}$ & $\best{10.9\%}$ & $\best{7.1\%}$ & $\best{0.8\%}$ \\

 
\noalign{\smallskip}
\hline
\noalign{\smallskip}
\end{tabular}
\end{center}
\setlength{\tabcolsep}{1.4pt}
\end{table*}

\subsubsection{Adaption for \tinyIN}
In order to adjust to the larger input size of \tinyIN,
we added an additional block to the architecture.
In order to keep the number of parameters similar,
we half the width parameter $w$ for this model, 
and set it to $32$. \medskip

\section{Empirical Results}
\subsection{Fitting the \nfun{}}

We plot the mean squared errors on the training set
for different activation functions
during the 1000 training epochs
in \Cref{fig:mse_figure}.
Furthermore, we visualize the actual functions
that were learned by the models in
\cref{fig:approximations}.
%
%

We observe that even though the \nfun is \ols, 
neither the \mm network nor the \abs network 
achieve low (training) loss, even when trained for 1000 epochs. 
In contrast, the \nact network 
achieves much lower loss,
as predicted by our theory.
We also show in \cref{fig:approximations}
that the function learned when using \maxmin or
\abs-networks is visibly different to the \nfun.

\subsection{Certified Robust Classification}
We show results where we compare our proposed \nact to the
\maxmin activation in \Cref{fig:results} as well as \Cref{table:results}.
We can see that it is in fact possible to replace the commonly used
\maxmin activation with our \nact without reducing the \cra.
In 7 out of 9 settings, 
the \cra (for $\epsilon=36/255)$
is actually slightly higher when using the \nact.
This shows that the \nact is in fact competitive to the \maxmin activation.
However, it also shows that the additional expressive power that comes with
the \nact does not improve performance on certified robust classification
by a larger amount, at least in the setup we consider.
This indicates that there might be other fundamental restrictions
of \ols networks, and the lack of expressiveness is not that
important in practice.

Interestingly, both cases where \maxmin performs better
than the \nact are with \SOC layers.
This might indicate that the best choice of activation function
does depend on the \ols layer used.

\begin{table}[t]
\setlength{\tabcolsep}{4pt}
\begin{center}
\caption{
    Test set accuracy and \cra for different thresholds $\epsilon$.
    We compare the different initialization strategies
    for the \nact for \AOL and \CPL layers on \CIFARS.
    Better results are marked in bold.
}
\label{table:ablation_results}
\begin{tabular}{ccrrrrr}
\hline\noalign{\smallskip}
Layer       & \nact             & Accuracy  & \multicolumn{4}{c}{\tableheader} \\ %
Type        & Initialization    &           & 
$\epsilon=\frac{36}{255}$ & $\epsilon=\frac{72}{255}$ & $\epsilon=\frac{108}{255}$ & $\epsilon=1$ \\%
            

\noalign{\smallskip} \hline \noalign{\smallskip}
\multirow{3}{*}{AOL} & AbsId & $\best{72.9\%}$ & $\best{60.7\%}$ & $\best{46.9\%}$ & $\best{34.7\%}$ & $\best{4.5\%}$ \\
 & Zero & $65.3\%$ & $51.2\%$ & $36.5\%$ & $24.4\%$ & $1.2\%$ \\
 & Random & $70.1\%$ & $56.8\%$ & $41.8\%$ & $29.1\%$ & $2.3\%$ \\
\noalign{\smallskip} \hline \noalign{\smallskip}
\multirow{3}{*}{CPL} & AbsId & $75.0\%$ & $\best{63.0\%}$ & $48.6\%$ & $35.6\%$ & $4.2\%$ \\
 & Zero & $\best{75.4\%}$ & $\best{63.0\%}$ & $\best{49.6\%}$ & $\best{36.2\%}$ & $\best{4.5\%}$ \\
 & Random & $73.1\%$ & $59.9\%$ & $45.2\%$ & $32.3\%$ & $2.8\%$ \\ 

 
\noalign{\smallskip}
\hline
\noalign{\smallskip}
\end{tabular}
\end{center}
\setlength{\tabcolsep}{1.4pt}
\end{table}

\begin{figure}[pt]
\centering
\includegraphics[width=8cm]{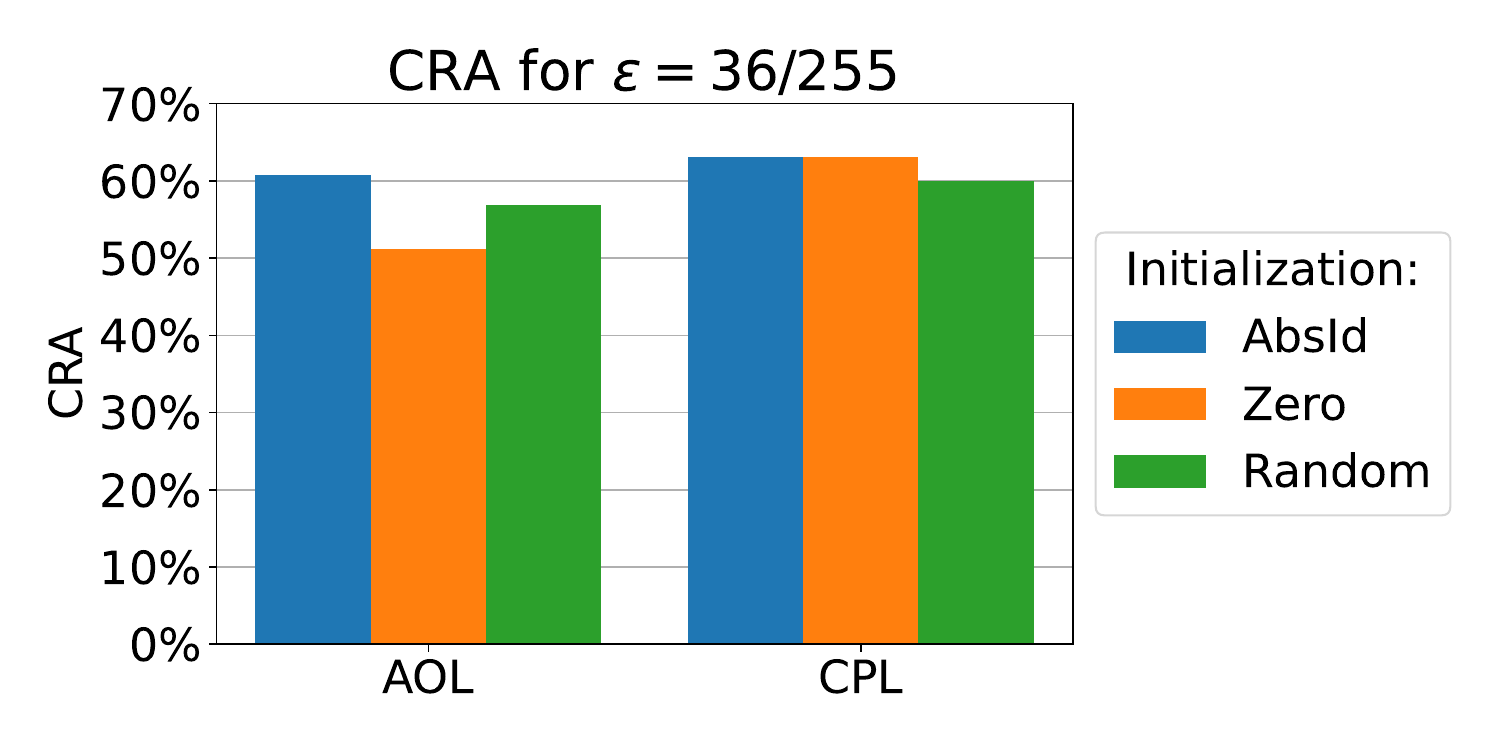}
\caption{
    \Cra for for different initialization strategies
    for the \nact.
    Evaluated on \CIFARS
    for \AOL and \CPL layers.
}
\label{fig:ablation_results}
\end{figure}

\subsection{Ablation Experiment}
We conducted an ablation experiment to evaluate how
important the initialization strategy for the
\nact is.
We compare 3 different ways of initializing:
The first way is the initialization described in \cref{sec:nact_init},
we refer to it as \emph{\absidinit}.
The second way is to initialize the \nact as the identity map,
by setting all parameters initially to zero
(\emph{\zeroinit}).
The third way, \emph{\randominit} 
is to choose the initialization values randomly.
We set $\theta_1 = - 10^{u_1}$ and $\theta_2 = + 10^{u_2}$,
for $u_1$ and $u_2$ randomly sampled in the interval $[-5, 0]$.

The results of this ablation study can be found in 
\cref{fig:ablation_results} and \cref{table:ablation_results}.
We see that the kind of initialization does influence the \cra a model achieves.
For \CPL, that is already a non-linear layer, 
\zeroinit{} performs compareable to \absidinit.
The \randominit{} does perform worse than \absidinit in both settings.

\section{Conclusion}
In this work, we analyze the expressive power of \ols networks. 
We show that the use of  previously proposed activation functions, such as \emph{MaxMin}, causes an unnecessary restriction of the class of functions that the network can represent, and we propose 
the \nact as a more expressive alternative. 
In particular, we prove that in the one-dimensional setting, $\mathcal{N}$-networks are \universal, in the sense that they can represent any piece-wise linear function with finitely many segments.
Our experiments show that $\mathcal{N}$-networks are not only theoretically appealing
but are also a reasonable replacement for the \maxmin activation in experiments,
reaching comparable 
\cra in standard benchmarks.

\bibliographystyle{IEEEtran}
\bibliography{references}

\begin{thebibliography}{10}
\providecommand{\url}[1]{#1}
\csname url@samestyle\endcsname
\providecommand{\newblock}{\relax}
\providecommand{\bibinfo}[2]{#2}
\providecommand{\BIBentrySTDinterwordspacing}{\spaceskip=0pt\relax}
\providecommand{\BIBentryALTinterwordstretchfactor}{4}
\providecommand{\BIBentryALTinterwordspacing}{\spaceskip=\fontdimen2\font plus
\BIBentryALTinterwordstretchfactor\fontdimen3\font minus \fontdimen4\font\relax}
\providecommand{\BIBforeignlanguage}[2]{{%
\expandafter\ifx\csname l@#1\endcsname\relax
\typeout{** WARNING: IEEEtran.bst: No hyphenation pattern has been}%
\typeout{** loaded for the language `#1'. Using the pattern for}%
\typeout{** the default language instead.}%
\else
\language=\csname l@#1\endcsname
\fi
#2}}
\providecommand{\BIBdecl}{\relax}
\BIBdecl

\bibitem{adversarial_examples_2014_szegedy}
C.~Szegedy, W.~Zaremba, I.~Sutskever, J.~Bruna, D.~Erhan, I.~Goodfellow, and R.~Fergus, ``Intriguing properties of neural networks,'' in \emph{International Conference on Learning Representations (ICLR)}, 2014.

\bibitem{explainable_requires_robustness_2021_Leino}
\BIBentryALTinterwordspacing
K.~Leino. (2021) Ai explainability requires robustness. [Online]. Available: \url{https://towardsdatascience.com/ai-explainability-requires-robustness-2028ac200e9a}
\BIBentrySTDinterwordspacing

\bibitem{adversarial_training_2015_goodfellow}
I.~Goodfellow, J.~Shlens, and C.~Szegedy, ``Explaining and harnessing adversarial examples,'' in \emph{International Conference on Learning Representations}, 2015.

\bibitem{randomized_smoothing_2019_cohen}
J.~Cohen, E.~Rosenfeld, and Z.~Kolter, ``Certified adversarial robustness via randomized smoothing,'' in \emph{International Conference on Machine Learing (ICML)}, 2019.

\bibitem{Cisse_2017_ICML}
M.~Ciss{\'e}, P.~Bojanowski, E.~Grave, Y.~N. Dauphin, and N.~Usunier, ``Parseval networks: Improving robustness to adversarial examples,'' in \emph{International Conference on Machine Learing (ICML)}, 2017.

\bibitem{arjovsky2017wasserstein}
M.~Arjovsky, S.~Chintala, and L.~Bottou, ``Wasserstein generative adversarial networks,'' in \emph{International Conference on Machine Learing (ICML)}, 2017.

\bibitem{Tsuzuku_2018_NIPS}
Y.~Tsuzuku, I.~Sato, and M.~Sugiyama, ``{L}ipschitz-margin training: Scalable certification of perturbation invariance for deep neural networks,'' in \emph{Conference on Neural Information Processing Systems (NeurIPS)}, 2018.

\bibitem{Anil_2019_ICML}
C.~Anil, J.~Lucas, and R.~B. Grosse, ``Sorting out {L}ipschitz function approximation,'' in \emph{International Conference on Machine Learing (ICML)}, 2019.

\bibitem{Li_2019_NIPS_BCOP}
Q.~Li, S.~Haque, C.~Anil, J.~Lucas, R.~B. Grosse, and J.-H. Jacobsen, ``Preventing gradient attenuation in {L}ipschitz constrained convolutional networks,'' \emph{Conference on Neural Information Processing Systems (NeurIPS)}, 2019.

\bibitem{Huang_2020_CVPR}
L.~Huang, L.~Liu, F.~Zhu, D.~Wan, Z.~Yuan, B.~Li, and L.~Shao, ``Controllable orthogonalization in training {DNNs},'' in \emph{Conference on Computer Vision and Pattern Recognition (CVPR)}, 2020.

\bibitem{Trockman_2021_ICLR}
A.~Trockman and J.~Z. Kolter, ``Orthogonalizing convolutional layers with the {Cayley} transform,'' in \emph{International Conference on Learning Representations (ICLR)}, 2021.

\bibitem{Singla_2021_ICML}
S.~Singla and S.~Feizi, ``Skew orthogonal convolutions,'' in \emph{International Conference on Machine Learing (ICML)}, 2021.

\bibitem{Leino_2021_ICML}
K.~Leino, Z.~Wang, and M.~Fredrikson, ``Globally-robust neural networks,'' in \emph{International Conference on Machine Learing (ICML)}, 2021.

\bibitem{Prach_2022_ECCV}
B.~Prach and C.~H. Lampert, ``Almost-orthogonal layers for efficient general-purpose {{L}ipschitz} networks,'' in \emph{European Conference on Computer Vision (ECCV)}, 2022.

\bibitem{Meunier_2022_ICML}
L.~Meunier, B.~J. Delattre, A.~Araujo, and A.~Allauzen, ``A dynamical system perspective for {{L}ipschitz} neural networks,'' in \emph{International Conference on Machine Learing (ICML)}, 2022.

\bibitem{brau_2023_ugnn}
F.~Brau, G.~Rossolini, A.~Biondi, and G.~Buttazzo, ``Robust-by-design classification via unitary-gradient neural networks,'' in \emph{Proceedings of the AAAI Conference on Artificial Intelligence}, 2023.

\bibitem{soc}
S.~Singla and S.~Feizi, ``Skew orthogonal convolutions,'' in \emph{International Conference on Machine Learing (ICML)}, 2021.

\bibitem{comparison_2023_unpublished}
B.~Prach, F.~Brau, G.~Buttazzo, and C.~H. Lampert, ``1-{L}ipschitz layers compared: Memory, speed, and certifiable robustness,'' 2023, unpublished.

\bibitem{bubeck_2021_universal}
S.~Bubeck and M.~Sellke, ``A universal law of robustness via isoperimetry,'' \emph{Advances in Neural Information Processing Systems}, 2021.

\bibitem{bombari_2023_beyond_universal_law}
S.~Bombari, S.~Kiyani, and M.~Mondelli, ``Beyond the universal law of robustness: Sharper laws for random features and neural tangent kernels,'' in \emph{International Conference on Machine Learing (ICML)}, 2023.

\bibitem{corner_2023_leino}
K.~Leino, ``Limitations of piecewise linearity for efficient robustness certification,'' \emph{arXiv preprint arXiv:2301.08842}, 2023.

\bibitem{universal_2019_cohen}
J.~E. Cohen, T.~Huster, and R.~Cohen, ``Universal {L}ipschitz approximation in bounded depth neural networks,'' \emph{arXiv preprint arXiv:1904.04861}, 2019.

\bibitem{dense_2020_Eckstein}
S.~Eckstein, ``{L}ipschitz neural networks are dense in the set of all {L}ipschitz functions,'' \emph{arXiv preprint arXiv:2009.13881}, 2020.

\bibitem{householder_2021_singla}
S.~Singla, S.~Singla, and S.~Feizi, ``Improved deterministic $l_2$ robustness on {CIFAR}-10 and {CIFAR}-100,'' in \emph{International Conference on Learning Representations (ICLR)}, 2021.

\bibitem{Splines_2022_ArXiv}
S.~Ducotterd, A.~Goujon, P.~Bohra, D.~Perdios, S.~Neumayer, and M.~Unser, ``Improving {L}ipschitz-constrained neural networks by learning activation functions,'' \emph{arXiv preprint arXiv:2210.16222}, 2022.

\bibitem{splines_2020_ieee}
S.~Aziznejad, H.~Gupta, J.~Campos, and M.~Unser, ``Deep neural networks with trainable activations and controlled {L}ipschitz constant,'' \emph{IEEE Transactions on Signal Processing}, 2020.

\bibitem{Neumayer_2022_ArXiv}
S.~Neumayer, A.~Goujon, P.~Bohra, and M.~Unser, ``Approximation of {L}ipschitz functions using deep spline neural networks,'' \emph{arXiv preprint arXiv:2204.06233}, 2022.

\bibitem{onecyclelr_2019_smith}
L.~N. Smith and N.~Topin, ``Super-convergence: Very fast training of neural networks using large learning rates,'' in \emph{Artificial intelligence and machine learning for multi-domain operations applications}, 2019.

\bibitem{randaugment_2020_cubuk}
E.~D. Cubuk, B.~Zoph, J.~Shlens, and Q.~V. Le, ``Randaugment: Practical automated data augmentation with a reduced search space,'' in \emph{Proceedings of the IEEE/CVF conference on computer vision and pattern recognition workshops}, 2020.

\end{thebibliography}

\appendices
\clearpage
\pagebreak 
\onecolumn

\section{Limitations of existing activation functions} 
\label{sec:maxminbad}
In this section, we will prove that common activation
functions such as ReLU, leaky ReLU, \abs{} and \mm{}
are not \universal.
Our first theorem and proof will be about \abs{}-networks.

\begin{lem} \label{lem:absvalbad}
    The Absolute Value activation is not \universal.
\end{lem}
\begin{proof}
\newcommand{\vfun}{\operatorname{v}} %
\newcommand{\wfun}{\operatorname{w}} %
In order to prove \Cref{lem:absvalbad}, it is enough to show that
there exists a 1-dimensional, \ols{}, \CPWL{} function
that can not be expressed by any \ols{} \abs-network.
There is in fact a simple function that can not be expressed,
we call it the \nfun. It is given as:
\begin{align}
    \mathcal{N}(x) = \left\{\begin{array}{llrll}
        x+1 & \text{ ... } &            & x  & \le -\half \\
        -x  & \text{ ... } & -\half \le & x  & \le  \half \\
        x-1 & \text{ ... } &  \half \le & x.
    \end{array} \right.
\end{align}
It is visualized in \cref{fig:n_function}.
We will show that \abs-networks can not express this function.

Suppose there is an \abs-network $f$ that can express the \nfun.
First, note that the \nfun is not linear,
so we need to apply an activation function at some point.
Consider the first time the \abs-activation is applied non-trivially.
We will consider the input and the output of this first activation,
and denote them by $\vfun(x) \in \mathbb{R}^k$
and $\wfun(x) \in \mathbb{R}^k$ 
(as a function of input $x \in \mathbb{R}$).
We obtain $\wfun(x)$ from $\vfun(x)$
by applying the \abs-activation to at least some of
the elements.
Furthermore, we can obtain the output of the network, $f(x)$,
from $\wfun(x)$ by applying
the remaining layers of the network.
The only assumption we make about those layer is that they are
all \ols.

By our assumption, we know that $\vfun$ is a linear function of the input,
write 
\begin{align}
    \vfun(x) = \vec{\theta} x + \vec{b},
\end{align}
for some $\vec{\theta}\in \mathbb{R}^k$ and $\vec{b} \in \mathbb{R}^k$.
Note that $\|\vec{\theta}\|_2 \le 1$ by the 1-Lipschitz property.

Now consider an index $t$ such that
$\wfun(x)_t = |\vfun(x)_t|$, and $\theta_t \ne 0$,
and define $x_0 \in \mathbb{R}$ so that $\vfun(x_0)_{t} = 0$.
Then we have that for any $\delta \in \mathbb{R}$:
\newcommand{\halfdelta}{\frac{\delta}{2}}
\newcommand{\wfp}{\wfun \left(x_0 + \halfdelta \right)}
\newcommand{\wfm}{\wfun \left(x_0 - \halfdelta \right)}
\newcommand{\vfp}{\vfun \left(x_0 + \halfdelta \right)}
\newcommand{\vfm}{\vfun \left(x_0 - \halfdelta \right)}
\newcommand{\xpd}{\left(x_0 + \halfdelta \right)}
\newcommand{\xmd}{\left(x_0 - \halfdelta \right)}
\begin{align}
    w \xpd_t
    = \left| v \xpd_t \right|
    = \left| \theta_t \frac{\delta}{2} \right|
    = \left| v \xmd_t \right|
    = w \xmd_t
\end{align}
and 
\begin{align}
    &\left\| \wfp - \wfm \right\|_2^2 \\
    &= \sum_{i=1}^k \left( \wfp_i - \wfm_i \right)^2 \\
    &= \sum_{i\ne t} \left( \wfp_i - \wfm_i \right)^2 \\
    &\le \sum_{i\ne t} \left( \vfp_i - \vfm_i \right)^2 \\
    &=\left\| \vfp - \vfm \right\|_2^2 
    - \left| \theta_t \delta \right|^2.
\end{align}
Then, for $f$ the output of the full network
we can use the Lipschitz property of the layers
in order to get the following:
\begin{align}
    &\left( f\xpd - f\xmd \right)^2 \\
    &\le \left\| \wfun \xpd - \wfun \xmd \right\|_2^2 \\
    &\le \left\| \vfun \xpd - \vfun \xmd \right\|_2^2 
        - \left| \theta_t \delta \right|^2 \\
    &\le \left\| \xpd - \xmd \right\|_2^2 
        - \left| \theta_t \delta \right|^2 \\
    &= \left( 1-\theta_t^2 \right) \delta^2.
\end{align}

By our assumption, $\theta_t \ne 0$, so the final term
is strictly smaller than $\delta^2$.

However, the \nfun{} has the property that
\begin{align}
    & \lim_{\delta \rightarrow \infty} \left( \frac{
        \mathcal{N}\xpd - \mathcal{N}\xmd
        }{\delta}\right)^2
    = 1,
\end{align}
and this contradicts that $f$ is equal to the \nfun.
\end{proof}

This result will allow us to prove one of our main theorems: %
\twopiecebad*
\begin{proof}
    Any 2-piece piecewise linear activation function
    can be written as a linear (and \ols) combination
    of an \abs-activation and the identity map
    (and some bias terms).
    Therefore, \Cref{thm:2piecebad} follows directly from \Cref{lem:absvalbad}.
\end{proof}
This corollary shows that activations such as ReLU
and leaky ReLU are not \universal{}.
Finally, we will also show that the commonly used
\mm-activation is not \universal.
\maxminbad*

\begin{proof}
    It has been shown before that \mm-networks can express
    exactly the same class of functions as \abs-networks
    on bounded inputs \citep{Anil_2019_ICML}.
    Similarly, on unbounded input, 
    any \abs-network
    can be expressed as a network consisting of \abs{} activations
    as well as identity connections.
One way of doing this is by expressing the \mm-activation
(similar to the construction in \cite{Anil_2019_ICML}) as
\begin{align} \label{eq:maxmin_equals_abs_id}
    \operatorname{MaxMin} (
    \left( \begin{matrix} x \\ x \end{matrix} \right) )
    = M \sigma ( M 
        \left( \begin{matrix} x \\ y \end{matrix} \right) ),
\end{align}
where
\begin{align}
    M = \frac{1}{\sqrt{2}} \left[ 
        \begin{matrix} 1 & 1 \\ 1 & -1  \end{matrix} \right]
    \quad \text{ and } \quad
    \sigma( \left( \begin{matrix} x \\ y \end{matrix} \right) )
    = \left( \begin{matrix} x \\ |y| \end{matrix} \right).
\end{align}
Therefore, any function that can be expressed by 
a \mm-network can also be expressed by a network
with \abs-activations (and identity connections),
and \Cref{thm:2piecebad} implies
that also \mm{} is not universal.
\end{proof}

Finally, activations such as 
Householder activations \citep{householder_2021_singla}
can also be written as a concatenation of rotations
an \mm-activations 
(see e.g. \cite{Splines_2022_ArXiv}).  
Therefore, these are not \universal either.

\section{Universality of the \titlenact{}} \label{sec:nactgood}

In this section we will restate and proof Theorem \ref{thm:nactuniversal} 
about the universality for the \nact{}:
\nactuniveral*

Recall the definition of the \nact{}:
\nactdef{}%
where $\thmax = \max\left(\theta_1, \theta_2\right)$
and $\thmin = \min\left(\theta_1, \theta_2\right)$.
For a visualization of the \nact see \Cref{fig:n_activation}.
Furthermore recall that an \expressable{} function is a \ols{}, \CPWL{} function.

\newcommand{\wgo}{with gradient $1$ before the first and after the last non-linearity}

We will prove two useful lemmas in order to prove \Cref{thm:nactuniversal}.
The first one is about expressing increasing functions:
\begin{lem}
    Any increasing, \expressable{}, \onedim{} function $f$
    \wgo{} and 
    with $k$ non-linearities
    can be expressed with $k$ \ols{}, linear layers
    and $(k-1)$ \nact{}s.
\end{lem}

\begin{proof}   
    Suppose $f$ has non-linearities at $t_1, \dots, t_k$.
    Then for $i=1, \dots, k$ we will construct a continuous function
    $g_i: \mathbb{R} \rightarrow \mathbb{R}$
    that has $i$ layers such that
    \begin{align}
        \begin{array}{rllll}
            g_i'(x) &=& f'(x) &\quad \text{ if } \quad  & x < t_i \\
            g_i(x)  &=& x     &\quad \text{ if } \quad  & x \ge t_i.
        \end{array}
    \end{align}
    Then up to a bias term $g_k = f$ 
    and we are done.

    In order to construct such functions $g_i$,
    we will  define $s_i$ to be the slope of $f$
    between $t_i$ and $t_{i+1}$, and furthermore
    define $\alpha_i$ and $\beta_i$ such that
    $\alpha_i^2 - \beta_i^2 = s_i$ and $\alpha_i^2 + \beta_i^2 = 1$.
    We also define $\sigma_i: \mathbb{R} \rightarrow \mathbb{R}$ 
    as the (element-wise) application of the \nact{}, 
    with certain parameters:
    \begin{align}
        \sigma_i \left( \begin{matrix} x_1 \\ x_2 \end{matrix} \right)
        = \left( \begin{matrix} x_1 \\
        \mathcal{N}(x_2; \beta_i t_i, \beta_i t_{i+1}) \end{matrix} \right).
    \end{align}
    Note that the first element is given as the identity,
    which is equal to the \nact{} with both parameters equal to $0$.

    Having defined $\alpha_i$, $\beta_i$ as well as $\sigma_i$,
    we can define $g_i$. Set $g_1(x) = x$ and
    \begin{align} \label{gdef}
        g_{i+1}(x) =
        \left( \begin{matrix} \alpha_i & \beta_i \end{matrix} \right)
        \sigma_i \left(
        \left( \begin{matrix} \alpha_i \\ \beta_i \end{matrix} \right)
        g_i(x)
        \right).
    \end{align}
    Note that in this definition we use 2 matrix multiplications,
    which seems to be contradicting the fact that $g_i$ has $i$
    layers. However, we can merge two adjacent matrix multiplications
    into one, allowing us to in fact express $g_i$ using $i$ layers.

    Furthermore, since $\alpha_i^2 + \beta_i^2 = 1$ holds, both
    the matrices have a spectral norm of at most one.
    To see this note that the matrix
     \begin{align} \label{alphabeta}
        \left( \begin{matrix} 
            \alpha_i & \beta_i \\ 
            \beta_i & -\alpha_i
        \end{matrix} \right)
    \end{align}
    is orthogonal, 
    and both matrices in \Cref{gdef} are a row
    or a column of the matrix in \Cref{alphabeta}.

    From the definition of $g_{i+1}$ we get that
    \begin{align}
        g_{i+1}(x)
        &= \alpha_i^2 g_i(x)
            + \beta_i \mathcal{N}\left(\beta_i g(x); 
            \beta_i t_i, \beta_i t_{i+1}\right) \\
        &= \alpha_i^2 g_i(x)
            + \beta_i^2 \left\{ \begin{array}{llrll}
                g_i(x) - 2 g_i(t_i)
                & \text{ .. } & & g(x) & \le t_i \\
                -g_i(x)
                & \text{ .. } & t_i \le & g(x) & \le t_{i+1} \\
                g_i(x) - 2 g_i(t_{i+1})
                & \text{ .. } & t_{i+1} \le & g(x), & \\
        \end{array} \right.
    \end{align}
    and therefore
    \begin{align}
        g_{i+1}'(x)
        &= \left\{ \begin{array}{lll}
            \left(\alpha_i^2 - \beta_i^2\right) g_i'(x)   & \text{ .. } & t_i \le x \le t_{i+1} \\
            \left(\alpha_i^2 + \beta_i^2\right) g_i'(x)   & \text{ .. } & \text{otherwise}.
        \end{array} \right. 
    \end{align}

    Noting that $\alpha_i^2 - \beta_i^2 = s_i$ 
    and $\alpha_i^2 + \beta_i^2 = 1$ completes the proof.
\end{proof}

We can use the lemma above as an inductive basis for 
the proof of the following lemma:
\begin{lem} \label{lem:gradient1outside}
    Any \expressable{}, \onedim{} function $f$
    \wgo{} and 
    with $k$ non-linearities and $2l$ extreme points
    can be expressed with $k$ \ols{}, linear layers
    and $(k + l - 1)$ \nact{}s.
\end{lem}

\begin{proof}
    We will proof this lemma by induction on $l$.
    If $l=0$, $f$ is increasing, so we proved this case
    in the previous lemma.
    
    Now suppose we want to fit a function $f$ 
    with $2l$ local extreme points.

    Define
    \begin{itemize}
        \item $k_s$ .. position of highest local maxima,
        \item $k_t$ .. position of lowest local minima with $k_t > k_s$.
    \end{itemize}
    Then, by the definition of $k_s$ and $k_t$, we have that
    \begin{itemize}
        \item $f(x) \le f(k_s)$ for $x \le k_s$,
        \item $f(k_t) \le f(x) \le f(k_s)$ for $k_s \le x \le k_t$, and
        \item $f(k_t) \le f(x)$ for $k_t \le x$.
    \end{itemize}

    We will define a function $g$, that is similar to $f$ but
    with the part between $k_s$ and $k_t$ flipped. Formally,
    \begin{align}
    g = \left\{ \begin{array}{lllll}
        f(x) - 2 f(k_s)     & \text{ .. } &         &x  &\le k_s \\
        - f(x)              & \text{ .. } & k_s \le  &x  &\le k_t \\
        f(x) - 2f(k_t)      & \text{ .. } & k_t \le  &x. \\
    \end{array} \right.
    \end{align}

    Then $g$ does not have local extreme points at $k_s$ and $k_t$,
    and it keeps all other local extreme points of $f$
    (and does not have any new ones).
    Therefore, $g$ has $2l-2$ local extreme points, 
    so by induction we can represent it by a network with $k$
    layers and $(k+(l-1)-1)$ \nact{}s.

    Furthermore, the properties we stated above for $f$ imply
    similar properties for $g$:
    \begin{itemize}
        \item $g(x) \le g(k_s)$ for $x \le k_s$,
        \item $g(k_s) \le g(x) \le g(k_t)$ for $k_s \le x \le k_t$, and
        \item $g(k_t) \le g(x)$ for $k_t \le x$
    \end{itemize}
    This implies that
    \begin{align}
        g(k_s) \le g(x) \le g(k_t)
        \quad \Longleftrightarrow \quad
        k_s \le x \le k_t.
    \end{align}
    Define a function $h$ as applying the \nact
    with certain parameters to the output of $g$:
    \begin{align}
        h(x) = \mathcal{N}\left(g(x); 
        g(k_s), g(k_t)
        \right).
    \end{align}
    Then
    \begin{align}
        h'(x) = \left\{\begin{array}{lll}
            -g'(x)  & \text{ .. } & k_s \le x \le k_t \\
            g'(x)   & \text{ .. } & \text{otherwise}.
        \end{array} \right.,
    \end{align}
    so $h'(x) = f'(x)$, and $h(x) = f(x)$ up to a bias term.
    Therefore, we know we can express $f$ by using $g$ and
    an additional \nact.
    So, by induction we know we can express $f$ using
    a network with $k$ layers and $(k+l-1)$ \nact{}s.
    This completes the proof.
\end{proof}

\begin{corr}
    Any \expressable{} function $f$ with bounded input
    and $k$ non-linearities
    can be expressed using $(k+2)$ linear \ols{} layers
    and $(\frac{3}{2} k+ 2)$ \nact{}s.
\end{corr}

\begin{proof}
    Consider the continuous function that agrees with $f$ on the
    bounded input interval, and has derivative $1$ outside of it.
    Call it $g$.
    This function has (at most) $(k+2)$ non-linearities.
    Therefore, by our previous theorem, we can express this function
    by a network with $(k+2)$ layers and 
    $(k + 2 + l -1)$ \nact{}s,
    where $2l$ is the number of local extreme points of $g$.
    Now note that $2l \le k+2$, since any local extreme point must
    be at a non-linearity. This completes the proof.
\end{proof}
With this, we proved that \nact{}s are \universal{} as long
as the inputs are bounded. 
We will extend this result below to show that in
fact any \expressable{}, \onedim{} function,
even with potentially unbounded inputs can be
expressed by a \nact-Network,
as long as we are allowed to use $3$ \abs-activations
(or set $\theta_1 = -\infty$).

\begin{lem}
    Any \ols{} \CPWL{} function $f$ with potentially
    unbounded input region can be expressed by a network
    of width 2,
    with $k+5$ linear \ols{} layers, 
    $(\frac{3}{2} k +2)$ \nact{}s
    as well as $3$ Absolute Value activations.
\end{lem}

Proofing this lemma will also finally proof
Theorem \ref{thm:nactuniversal}, since setting
$\theta_1 = -\infty$ and $\theta_2 = 0$ 
makes an \nact{} identical to an \abs{}-Activation.

\begin{proof}
    For this theorem, the gradient of $f$ outside the regions where
    all the non-linearities lie will be important.
    Define $s_s$ and $s_t$ as the gradient of $f$ before the first
    and after the last non-linearity.

    As a first case, suppose $s_s \ge 0$ and $s_t \ge 0$.
    Note that either (or both) of them being $0$ is a bit of a special case,
    but we have designed out proof so that it also works in this case.
    Now if $s_s > 0 $, define $v_s$ such that $v_s \le t_1$ and
    $f(v_s) = \min_{1 \le i \le k} f(t_i)$. 
    This is possible because $s_s > 0$.
    With this definition
    $f(x) \le f(v_s) \Longleftrightarrow x \le v_s$.
    If $s_s = 0$, set $v_s = t_1$.
    Similarly, if $s_t > 0$, define $v_t$ such that $v_t \ge t_k$,
    and furthermore $f(v_t) = \max_{1 \le i \le k} f(t_i)$.
    If $s_t = 0$, set $v_t = t_k$.

    Now define $g$ as the (continuous) function that has $g(x) = f(x)$ 
    as long as $v_s \le x \le v_t$,
    and $g$ has slope $1$ otherwise.
    By Lemma \ref{lem:gradient1outside} 
    (and the fact that $2l \le k$)
    we can express $g$ using $(k+2)$ layers 
    and $(\frac{3}{2}k + 2)$ \nact{}s.

    Further define $h$ such that $h(x) = x$ as long as $v_s \le x \le v_t$,
    and $h$ has slope $s_s$ and $s_t$ outside this region.

    Then, note that $f(x) = g(h(x))$,
    as
    they agree when $v_s \le x \le v_t$,
    and the derivative of $f$ and $g \circ h$ also agrees outside that region.

    It turns out that $h$ itself can be expressed using $3$ layers
    and $2$ Absolute Value activations. 
    For example, we can use a linear combination of $x$ 
    and $-|( x - v_t )|$ in order to obtain a function
    with slope 1 if $x \le v_t$ and slope $s_t$ otherwise,
    and similarly change the slope when $x \le v_s$ to $s_s$.
    The coefficients are defined similarly to Equation \ref{gdef}.
    This completes the proof for the case where $s_s \ge 0$
    and $s_t \ge 0$.
    
    The case $s_s \le 0$ and $s_t \le 0$ can be reduced to the one above
    by changing the sign in the first layer.

    Finally, there are the cases that either 
    $s_s < 0$ and $s_t > 0$
    or $s_s > 0$ and $s_t < 0$. 
    Again, by a sign change we only need
    to consider one of them, say $s_s > 0$ and $s_t < 0$.
    In this case $f$ does have a global minimum at one of the non-linearities,
    call this point $v_m$
    We consider a function $g$ that has $g'(x) = f'(x)$ when $x \le v_m$,
    and $g'(x) = -f'(x)$ otherwise.
    The case above us tells us that $g$ can be expressed,
    and applying an Absolute Value Activation to the output of $g$
    (and a suitable bias before and after that) allows us to express $f$.

    Q.E.D.
\end{proof}

\end{document}